\documentclass{article} % For LaTeX2e
\usepackage[preprint]{colm2026_conference}

\usepackage{microtype}
\usepackage{hyperref}
\usepackage{url}
\usepackage{booktabs}
\usepackage{graphicx}
\usepackage{subcaption}
\usepackage{multirow}
\usepackage{xcolor}
\usepackage{colortbl}
\usepackage{amsthm}
\usepackage{amsmath}
\usepackage{amssymb}
\usepackage{mathtools}
\usepackage{bm}
\usepackage{enumitem}
\usepackage{float}
\usepackage[ruled,vlined,linesnumbered]{algorithm2e}
\usepackage{tikz}
\usetikzlibrary{arrows.meta,positioning,calc,fit,backgrounds,shapes.geometric}
\usepackage{cleveref}

\newtheorem{proposition}{Proposition}

\newcommand{\R}{\mathbb{R}}
\newcommand{\E}{\mathbb{E}}
\newcommand{\bX}{\mathbf{X}}
\newcommand{\bQ}{\mathbf{Q}}
\newcommand{\bK}{\mathbf{K}}
\newcommand{\bV}{\mathbf{V}}
\newcommand{\bP}{\mathbf{P}}
\newcommand{\bM}{\mathbf{M}}
\newcommand{\bA}{\mathbf{A}}
\newcommand{\bY}{\mathbf{Y}}
\newcommand{\bone}{\mathbf{1}}
\newcommand{\cN}{\mathcal{N}}
\newcommand{\cS}{\mathcal{S}}
\DeclareMathOperator{\softmax}{softmax}
\DeclareMathOperator{\sigmoid}{sigmoid}

\definecolor{swacolor}{HTML}{4A90D9}
\definecolor{stocolor}{HTML}{E74C3C}
\definecolor{fuscolor}{HTML}{2ECC71}
\definecolor{gatecolor}{HTML}{F39C12}
\definecolor{lightbg}{HTML}{F5F5F5}

\usepackage{lineno}

\definecolor{darkblue}{rgb}{0, 0, 0.5}
\hypersetup{colorlinks=true, citecolor=darkblue, linkcolor=darkblue, urlcolor=darkblue}

\title{Stochastic Attention: Connectome-Inspired Randomized Routing for Expressive Linear-Time Attention}

\author{Zehao Jin\\
Tsinghua University\\
\texttt{lunamos.thu@gmail.com} \\
\And
Yanan Sui \\
Tsinghua University \\
\texttt{ysui@tsinghua.edu.cn} \\
}

\begin{document}

\ifcolmsubmission
\linenumbers
\fi

\maketitle

\begin{abstract}

The whole-brain connectome of a fruit fly comprises over 130K neurons connected with a probability of merely 0.02\%, yet achieves an average shortest path of only 4.4 hops. Despite being highly structured at the circuit level, the network's long-range connections are broadly distributed across brain regions, functioning as stochastic shortcuts that enable efficient global communication. Inspired by this observation, we propose Stochastic Attention (SA), a drop-in enhancement for sliding-window attention (SWA) that applies a random permutation to the token sequence before windowed attention and restores the original order afterward. This transforms the fixed local window into a stochastic global one within the same $O(nw)$ per-layer budget. Through depth, independently sampled permutations yield exponentially growing receptive fields, achieving full sequence coverage in $O(\log_w n)$ layers versus $O(n/w)$ for SWA. We validate SA in two settings: pre-training language models from scratch, where a gated SA + SWA combination achieves the best average zero-shot accuracy, and training-free inference on Qwen3-8B and Qwen3-30B-A3B, where SA consistently outperforms SWA and matches or exceeds Mixture of Block Attention at comparable compute budgets. These results suggest that connectome-inspired stochastic routing is a practical primitive for improving the expressivity of efficient attention, complementary to existing linear and sparse approaches.
\end{abstract}

\section{Introduction}

How should an efficient attention mechanism route information? A compelling answer comes from neuroscience. The whole-brain connectome of fruit fly (\textit{Drosophila melanogaster}) \citep{Lin2024a,Dorkenwald2024a} reveals a network of ${\sim}$130{,}000 neurons with a connection probability of merely $0.02\%$, yet an average path length of only ${\sim}$4.4 hops and a small-worldness coefficient of ${\sim}$141. The Drosophila connectome is highly structured, featuring rich-club organization, elevated reciprocity, and selective motif over-representation \citep{Lin2024a}. Yet it also exhibits small-world topology: dense local clustering coexists with broadly distributed long-range connections. From any local neighborhood's perspective, the targets of these long-range projections resemble stochastic shortcuts scattered across brain regions. This suggests a design principle: global information flow can emerge from the interplay of structured local computation and distributed long-range shortcuts accumulated over a few synaptic steps.

This principle contrasts sharply with sliding-window attention (SWA) \citep{Beltagy2020,Jiang2023,Liu2021}, which restricts each token to a local window of size $w$ at $O(nw)$ cost per layer. SWA has been widely adopted in production models: Mistral \citep{Jiang2023} uses it throughout, while Gemma~2 \citep{Gemma2024} and gpt-oss \citep{OpenAI2025} alternate SWA with full attention. However, SWA’s deterministic locality limits the receptive field to $\ell w$ after $\ell$ layers, leaving large portions of the sequence unreachable when $w \ll n$. Existing remedies introduce global tokens \citep{Beltagy2020}, hand-crafted sparse patterns \citep{Zaheer2021}, or block-level routing \citep{Lu2025}, each adding architectural complexity.

Inspired by this organization, we propose \textbf{Stochastic Attention (SA)}: before applying windowed attention, we randomly permute the token sequence, and after attention, we restore the original order.
In the permuted space, the fixed local window spans a random subset of the full sequence, giving each token a uniform probability of attending to any other regardless of distance.
Through depth, independently sampled permutations yield exponentially growing receptive fields.
When combined with SWA via a learned gate, SA + SWA reproduces the connectome’s small-world regime: structured local clustering from SWA and distributed long-range shortcuts from SA.
The mechanism adds no learnable parameters to the attention itself and only $O(n)$ index-permutation overhead, implemented as simple permutation operations around any existing SWA kernel.

\begin{figure}[htbp]
  \begin{center}
    \centerline{\includegraphics[width=\columnwidth]{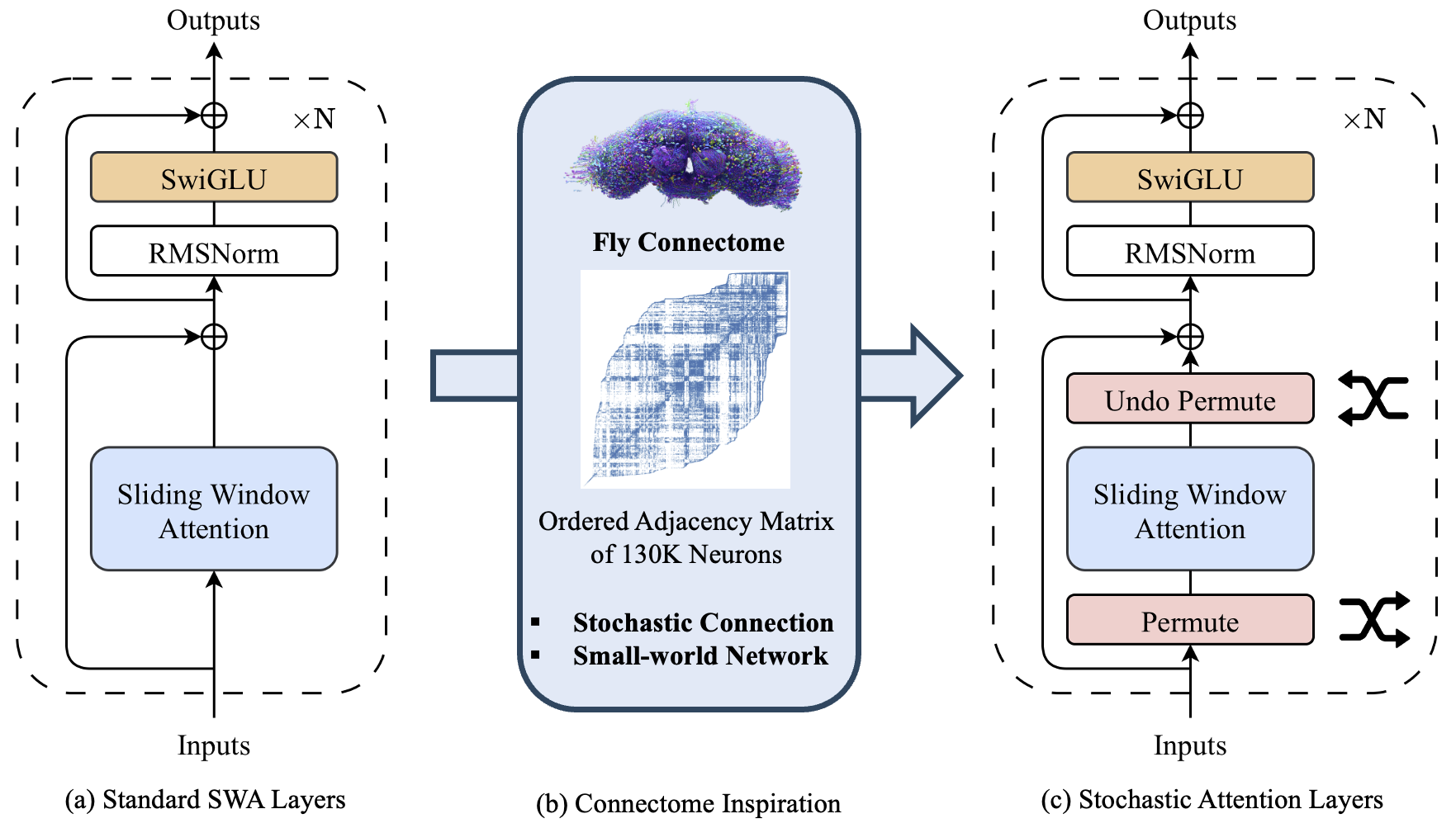}}
    \caption{
      Overview of Stochastic Attention (SA). (a)~A standard SWA Transformer layer. (b)~The fruit fly whole-brain connectome: the adjacency matrix, shown after Reverse Cuthill--McKee reordering to expose block structure, lacks clear diagonal blocks, indicating that connections are broadly distributed across brain regions rather than confined to local modules. (c)~An SA layer: token sequences are randomly permuted before windowed attention and restored afterward, producing stochastic long-range shortcuts analogous to the cross-regional connections in (b).
    }
    \label{fig:arch}
  \end{center}
\end{figure}

We evaluate SA in two complementary settings.
First, we pre-train language models from scratch, comparing SA, SWA, and their gated combination under identical architectures and training recipes.
The combined SA + SWA model achieves the best average zero-shot accuracy, demonstrating that the two mechanisms are complementary: SWA provides local coherence while SA provides global coverage.
Second, we apply SA as a training-free attention replacement in Qwen3-8B and Qwen3-30B-A3B \citep{Yang2025b}, where it consistently outperforms SWA and matches or exceeds MoBA \citep{Lu2025} at comparable compute budgets, demonstrating that stochastic routing is effective even when applied post-hoc to models trained with full attention.

\paragraph{Contributions.}
(1) We introduce \textbf{Stochastic Attention (SA)}, a parameter-free enhancement for SWA that randomly permutes token order before windowed attention, achieving exponential receptive field growth ($O(\log_w n)$ full coverage) within the same $O(nw)$ budget.
(2) We propose a \textbf{gated SA + SWA} combination that reproduces the connectome's small-world regime (local clustering from SWA, stochastic long-range shortcuts from SA) and provides theoretical analysis of coverage depth, pairwise connectivity, and bias-variance trade-offs.
(3) Experiments on pre-training (360M) and training-free inference (Qwen3-8B, Qwen3-30B-A3B) show SA + SWA consistently outperforms SWA and matches or exceeds full attention and MoBA at comparable compute.

\section{Related Work}

\paragraph{Windowed, sparse, and linear attention.}
Longformer \citep{Beltagy2020} augments local windows with global tokens. Swin Transformer \citep{Liu2021} uses shifted windows for cross-window interaction in vision. BigBird \citep{Zaheer2021} combines local, random, and global connections with expressivity guarantees. MoBA \citep{Lu2025} routes each query to the top-$k$ most relevant KV blocks.
Linear attention replaces softmax with kernelized or recurrent formulations \citep{Katharopoulos2020}, and Gated Linear Attention \citep{Yang2024} adds data-dependent gating for improved expressivity. Further advances include \citet{Yang2025,Dao2024,Oren2024,Guo2025,Lei2025}.
SA is complementary: it does not alter the attention formulation or define a sparse pattern, but changes which tokens become local neighbors across layers via random permutations, enabling global mixing within any existing windowed or linear attention kernel.

\paragraph{Token shuffling and rearrangement.}
Several vision methods employ deterministic token rearrangement to improve efficiency.
Shuffle Transformer \citep{Huang2021} permutes tokens across spatial windows using a fixed pattern inspired by channel shuffle, enabling cross-window information flow.
Token-Shuffle \citep{Ma2025} merges local visual tokens along the channel dimension (a spatial-to-depth reshape) to reduce token count in autoregressive image generation.
DeepStack \citep{Meng2024} distributes visual tokens across different Transformer layers rather than concatenating them all at the input.
These methods use structured, deterministic rearrangements specific to vision architectures.
SA differs in two key respects: the permutations are random and resampled per layer, which provides provable coverage guarantees ($O(\log_w n)$ depth), and the mechanism is modality-agnostic, applying directly to sequential language modeling.

\section{Method}
\label{sec:method}

We first introduce notation and background (\S\ref{sec:prelim}), then present the biological motivation (\S\ref{sec:connectome}), the SA mechanism (\S\ref{sec:stochastic_attention}), and the gated SA + SWA combination (\S\ref{sec:fusion}).

% -------------------------------------------------------------------------
\subsection{Preliminaries}
\label{sec:prelim}
% -------------------------------------------------------------------------

Consider an input sequence $\bX = (x_1, x_2, \ldots, x_n) \in \R^{n \times d}$, where $n$ is the sequence length and $d$ is the hidden dimension. Standard multi-head attention computes, for each head, the query, key, and value projections $\bQ = \bX W_Q$, $\bK = \bX W_K$, $\bV = \bX W_V \in \R^{n \times d_h}$, where $d_h = d / H$ and $H$ is the number of heads.

For each position $i \in [n]$, sliding window attention (SWA) restricts the attention to a local neighborhood $\cN_w(i)$ of size $w$. For the theoretical analysis, we use a symmetric circular window $\cN_w(i) = \{j \in [n] : |i - j|_n < w/2\}$, where $|i-j|_n = \min(|i-j|, n-|i-j|)$ denotes circular distance.\footnote{In practice, causal language models use a one-sided window $\cN_w(i) = \{j : 0 \leq i - j \leq w-1\}$. The theoretical results hold under either convention.} The SWA output is:
\begin{equation}
\label{eq:swa}
\mathrm{SWA}(i) = \sum_{j \in \cN_w(i)} \alpha_{ij} \, V_j, \quad \alpha_{ij} = \frac{\exp(Q_i^\top K_j / \sqrt{d_h})}{\sum_{k \in \cN_w(i)} \exp(Q_i^\top K_k / \sqrt{d_h})}.
\end{equation}
SWA achieves $O(nw)$ time and memory complexity, but its effective receptive field is limited to a linear growth of $\ell w$ after $\ell$ layers.

% -------------------------------------------------------------------------
\subsection{From Connectome to Stochastic Attention}
\label{sec:connectome}
% -------------------------------------------------------------------------

The fruit fly connectome comprises ${\sim}130{,}000$ neurons with connection probability $p\approx 0.02\%$ and average degree $\bar{k}\approx 21$, yet exhibits a short average path length of ${\sim}4.4$ hops, clustering coefficient ${\sim}0.048$, and small-worldness ${\sim}141$ \citep{Lin2024a}. The network is highly structured, but its short paths require broadly distributed long-range connections that, from any local neighborhood's perspective, function as stochastic shortcuts \citep{Watts1998}. Neither SWA (high clustering, diameter $\Theta(n/w)$) nor a random graph (short paths, negligible clustering) can achieve this regime alone.

We formalize this by modeling attention as a graph on $n$ tokens. In SA, a random permutation $\sigma_\ell\sim\mathrm{Uniform}(\cS_n)$ is drawn independently at each layer, and token $i$ attends to $\sigma_\ell^{-1}(\cN_w(\sigma_\ell(i)))$. The pairwise connection probability is (see Appendix~\ref{sec:appendix_proofs}):
\begin{equation}\label{eq:connection-prob}
  \Pr\bigl[j \in \sigma^{-1}(\cN_w(\sigma(i)))\bigr]
  = \frac{w-1}{n-1}\;\approx\;\frac{w}{n}\,,
\end{equation}

producing approximately uniform edges over all token pairs, analogous to the connectome's distributed long-range shortcuts. The gated SA + SWA combination thus mirrors the Watts--Strogatz construction: SWA preserves local clustering, SA adds distributed shortcuts.

Through multi-layer composition, the reachable set grows as $\E[|R_{\ell}(i)|] = \Omega(w^\ell)$ (Appendix~\ref{sec:proof_expansion}), giving full coverage in $O(\log_w n)$ layers vs.\ $O(n/w)$ for SWA. With $n\approx 130{,}000$ and $\bar{k}\approx 21$, this predicts $\lceil\log_{21}130{,}000\rceil = 4$ layers for all-pairs reachability, matching the connectome's mean path length of ${\sim}4.4$ \citep{Lin2024a}. \Cref{fig:method_analysis} illustrates these properties.

\begin{figure}[htbp]
  \centering
  \includegraphics[width=0.8\columnwidth]{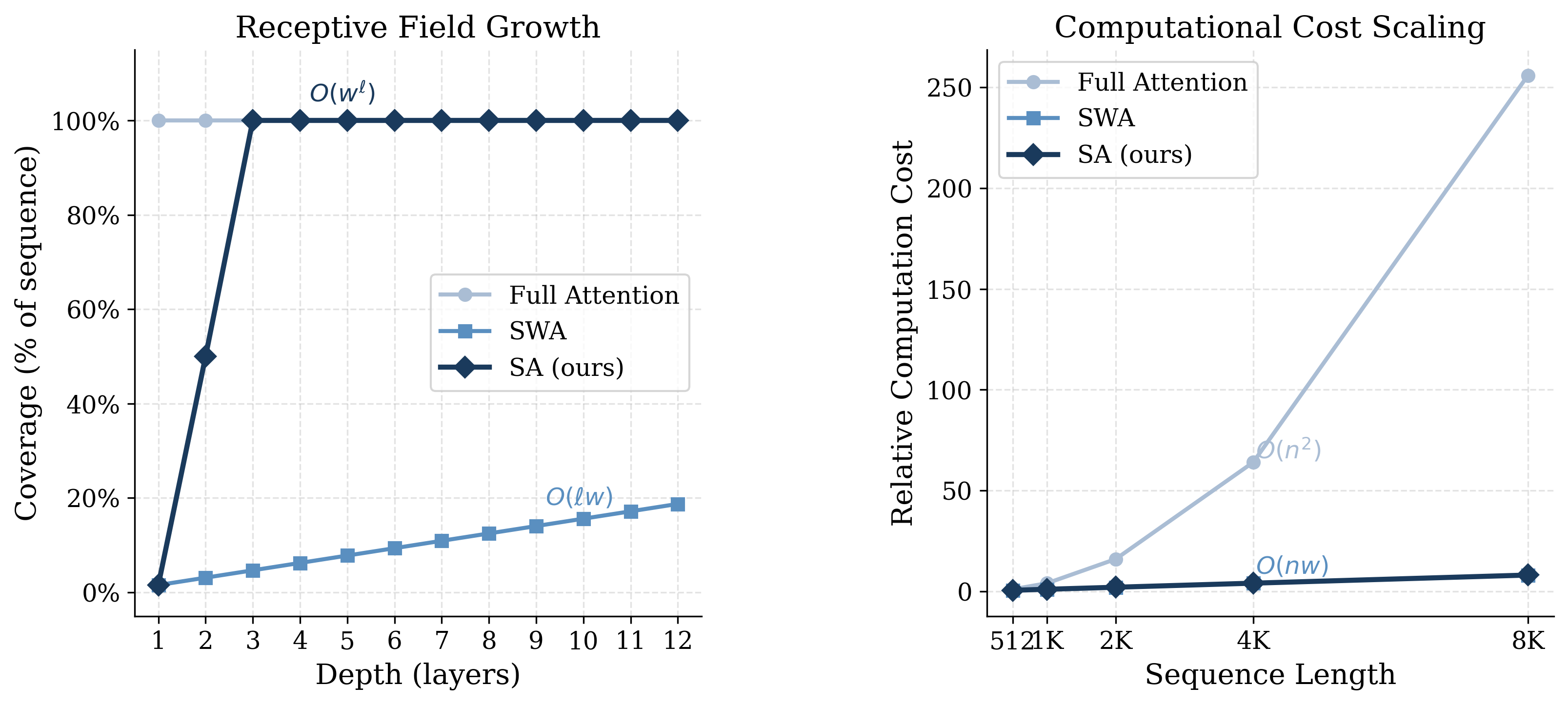}
  \caption{
    Left: Receptive field coverage as a function of depth ($n{=}2048$, $w{=}32$). SA achieves full sequence coverage in $O(\log_w n)$ layers via exponential growth, while SWA requires $O(n/w)$ layers with linear growth. Right: Computational cost scaling with sequence length ($w{=}256$). Both SA and SWA maintain $O(nw)$ linear scaling, while full attention grows quadratically.
  }
  \label{fig:method_analysis}
\end{figure}

% -------------------------------------------------------------------------
\subsection{Stochastic Attention}
\label{sec:stochastic_attention}
% -------------------------------------------------------------------------

The core idea is to apply a random permutation to the token sequence before performing sliding window attention, and then restore the original order afterward. This transforms the positionally local attention pattern into a stochastic global one.

Concretely, let $\sigma \sim \mathrm{Uniform}(\cS_n)$ be a random permutation drawn uniformly from the symmetric group $\cS_n$, and let $\bP_\sigma \in \{0,1\}^{n \times n}$ be the corresponding permutation matrix. Stochastic Attention operates in three stages:

\begin{enumerate}[leftmargin=2em,itemsep=3pt]
    \item \textbf{Permute.} Rearrange all representations: $\tilde{\bQ} = \bP_\sigma \bQ$, $\tilde{\bK} = \bP_\sigma \bK$, $\tilde{\bV} = \bP_\sigma \bV$.
    \item \textbf{Windowed Attention.} Apply standard SWA in permuted space: $\tilde{\bY} = \mathrm{SWA}(\tilde{\bQ}, \tilde{\bK}, \tilde{\bV}; w)$.
    \item \textbf{Undo Permute.} Restore original order: $\bY^{\mathrm{sto}} = \bP_{\sigma^{-1}} \tilde{\bY}$.
\end{enumerate}

In the original token space, position $i$ now attends to the random neighborhood
\begin{equation}
\label{eq:stochastic_neighborhood}
\tilde{\cN}_w^\sigma(i) = \big\{j \in [n] : |\sigma(j) - \sigma(i)|_n < w/2 \big\},
\end{equation}
which is a random subset of $[n]$ of expected size $w$, uniformly spread across the full sequence regardless of the original distance $|i - j|$. Equivalently, the mechanism is characterized by a binary random mask $\bM^\sigma \in \{0,1\}^{n \times n}$ with $M^\sigma_{ij} = \mathbb{1}[|\sigma(i) - \sigma(j)|_n < w/2]$, and the full operation can be written compactly as:
\begin{equation}
\label{eq:compact}
\bY^{\mathrm{sto}} = \softmax\!\Big(\frac{\bQ \bK^\top}{\sqrt{d_h}} \odot \bM^\sigma + (\bone - \bM^\sigma) \cdot (-\infty)\Big) \bV.
\end{equation}

The permutation $\sigma$ is sampled independently per layer and shared across all attention heads within that layer. During inference, $\sigma$ can be either freshly sampled (stochastic mode) or fixed to a predetermined permutation (deterministic mode). We use stochastic mode throughout our experiments.

In autoregressive language models, each token $i$ may only attend to tokens $j \leq i$. Under SA, this causal constraint is applied after permutation: in the permuted space, token $\sigma(i)$ attends to $\{j' \in \cN_w(\sigma(i)) : \sigma^{-1}(j') \leq i\}$. The effective neighborhood in the original space thus consists of tokens that are both within the permuted window and causally accessible. This preserves the autoregressive property while still enabling stochastic long-range connections. The connection probability in Eq.~\ref{eq:connection-prob} becomes approximately $\frac{w-1}{2(n-1)}$ on average (since roughly half of window neighbors are causally masked), which does not change the asymptotic $O(\log_w n)$ coverage depth.

The permute and undo-permute operations are $O(n)$ index rearrangements. The SWA computation in the permuted space costs $O(nw)$, identical to standard SWA. In practice, both steps are implemented via in-place index gather/scatter operations on GPU, which fuse naturally with FlexAttention \citep{dong2024}: the forward permutation is realized as \texttt{Q[sigma], K[sigma], V[sigma]} and the inverse as \texttt{Y[sigma\_inv]}, where both $\sigma$ and $\sigma^{-1}$ are precomputed as integer index tensors. The entire Stochastic Attention layer is thus a thin wrapper around any existing SWA implementation with negligible overhead. Pseudocode is provided in \Cref{alg:stochastic_attn} (Appendix~\ref{sec:appendix_arch}).

% -------------------------------------------------------------------------
\subsection{Combining SA and SWA}
\label{sec:fusion}
% -------------------------------------------------------------------------

As discussed in \S\ref{sec:connectome}, the fruit fly connectome achieves its small-world property through the coexistence of dense local connectivity and sparse long-range shortcuts.
Pure SA provides the shortcuts but disrupts locality: the random permutation scatters positionally adjacent tokens, and when $n \gg w$ the probability that two neighboring tokens share a stochastic window drops to $w/n \ll 1$.
To recover the small-world regime (high clustering and short paths), we combine SA and SWA in a dual-path architecture with learned attention gates:
\begin{equation}
\label{eq:fusion}
\bY = g^{\mathrm{sa}} \odot \bY^{\mathrm{sa}} + g^{\mathrm{swa}} \odot \bY^{\mathrm{swa}},
\end{equation}
where $\bY^{\mathrm{sa}}$ is the output of Stochastic Attention, $\bY^{\mathrm{swa}}$ is the output of standard SWA, and $g^{\mathrm{sa}}, g^{\mathrm{swa}} \in \R^{n \times d}$ are per-token, per-dimension gating weights.

Each gate is computed from its corresponding attention output via an independent sigmoid projection:
\begin{equation}
\label{eq:gate}
g^{\mathrm{swa}} = \sigmoid(W_g^{\mathrm{swa}}\, (\bY^{\mathrm{swa}})^\top)^\top, \quad g^{\mathrm{sa}} = \sigmoid(W_g^{\mathrm{sa}}\, (\bY^{\mathrm{sa}})^\top)^\top,
\end{equation}
where $W_g^{\mathrm{swa}}, W_g^{\mathrm{sa}} \in \R^{d \times d}$ are learnable parameters. Unlike a softmax gate that enforces $g^{\mathrm{sa}}_i + g^{\mathrm{swa}}_i = \bone$, the two sigmoid gates are independent, allowing the model to up-weight or down-weight both paths simultaneously. This design mirrors the single-path attention gate used in the non-fusion variants (see Appendix~\ref{sec:appendix_arch}).

Both attention paths run in parallel. The total cost is $O(nw)$ for SWA $+ O(nw)$ for SA $+ O(nd)$ for the gating projections, giving $O(nw + nd)$ overall. Since both $d$ and $w$ are constants with respect to $n$, the per-layer complexity remains $O(n)$. Pseudocode is provided in \Cref{alg:fusion} (Appendix~\ref{sec:appendix_arch}).

We additionally show that SA is an approximately unbiased estimator of uniform full attention (bias $O(1/w)$, variance $O(B^2/w)$), and that the gated SA + SWA combination admits a bias-variance decomposition where the gate learns to balance SWA's systematic bias against SA's stochastic variance. While a single SA layer has the same spectrum as SWA (permutation is a similarity transform), multi-layer composition with independent permutations breaks this similarity and yields rapid mixing consistent with the $O(\log_w n)$ receptive field bound. Full theoretical analysis, proofs, and a comparison table are provided in Appendix~\ref{sec:appendix_proofs}.

\section{Experiments}

We evaluate Stochastic Attention in two complementary settings.
First, we pre-train language models (${\sim}$360M parameters) from scratch to assess whether SA can close the expressivity gap between SWA and full attention (\S\ref{sec:pretraining}).
Second, we apply SA as a training-free attention replacement in Qwen3-8B and Qwen3-30B-A3B to test whether stochastic routing benefits pretrained models without retraining (\S\ref{sec:qwen}).
We conclude with an efficiency analysis (\S\ref{sec:efficiency}).

\subsection{Pre-training: language modeling}
\label{sec:pretraining}

Following the training recipe of \citep{Yang2024}, we train ${\sim}$360M-parameter decoder-only Transformers on a 6B-token subset of SlimPajama \citep{cerebras2023slimpajama} for 2.5 epochs (${\sim}$15B tokens) with 24 layers, $d{=}1024$, 16 heads, $w{=}256$, and sequence length 2048.
We compare four attention variants: Full Attention, SWA, SA, and SA + SWA.\footnote{All single-path variants (Full, SWA, SA) have 360M parameters. SA+SWA adds one extra gate ($d \times d$ per layer, ${\sim}$25M total, 385M overall). Full training details and hyperparameters are provided in Appendix~\ref{sec:appendix_arch} and Appendix~\ref{sec:appendix_training}.}
All models are evaluated zero-shot on WikiText~\citep{merity2016}, LAMBADA~\citep{paperno2016}, PIQA~\citep{bisk2019}, HellaSwag~\citep{zellers2019}, WinoGrande~\citep{sakaguchi2019}, and ARC-Easy~\citep{clark2018}.

\begin{table}[htbp]
  \centering
  \resizebox{\columnwidth}{!}{%
  \begin{tabular}{l|ccc|cccccc|c}
  \toprule
  \textbf{Model} & \textbf{Wiki.$\downarrow$} & \textbf{LMB-o$\downarrow$} & \textbf{LMB-s$\downarrow$} & \textbf{LMB-o} & \textbf{LMB-s} & \textbf{PIQA} &
  \textbf{Hella.} & \textbf{Wino.} & \textbf{ARC-e} & \textbf{Avg.} \\
  \midrule
  Full Attention         & \textbf{51.34} & 185.3 & 469.2 & 19.5 & 15.4 & \textbf{59.9} & 27.2 & \underline{51.3} & 35.8 & 34.9 \\
  SWA ($w{=}256$)        & 57.05 & \underline{156.1} & \textbf{370.6} & \underline{21.3} & \underline{17.0} & \underline{59.6} & \textbf{27.6} & 48.9 &
  \underline{36.0} & \underline{35.1} \\
  SA ($w{=}256$)         & 75.83 & 260.1 & 785.9 & 20.1 & 14.1 & 59.3 & 26.4 & \textbf{51.9} & 34.2 & 34.3 \\
  SA + SWA ($w{=}256$)   & \underline{51.98} & \textbf{131.7} & \underline{371.6} & \textbf{22.8} & \textbf{17.6} & \underline{59.6} & \underline{27.5} &
  50.5 & \textbf{37.5} & \textbf{35.9} \\
  \bottomrule
  \end{tabular}
  }
\caption{Zero-shot evaluation of language models trained on SlimPajama (15B tokens). All models share identical training setup, differing only in the attention
  mechanism. Wiki.\ and LMB ppl report perplexity ($\downarrow$), all others report accuracy ($\uparrow$). Best in \textbf{bold}, second best \underline{underlined}.}
\label{tab:pretraining_main}
  \end{table}

Table~\ref{tab:pretraining_main} reports zero-shot results.
The gated SA + SWA combination achieves the best average downstream accuracy (35.9) and the best LAMBADA scores (ppl 131.7, acc 22.8/17.6), while matching Full Attention in WikiText perplexity (51.98 vs.\ 51.34).
Pure SA alone suffers substantially higher perplexity than SWA (75.83 vs.\ 57.05), confirming that local coherence from fixed windows is essential for language modeling.
However, SA retains competitive downstream accuracy (avg 34.3 vs.\ SWA's 35.1), suggesting that stochastic global routing captures complementary information.
The SA + SWA fusion recovers the best of both: SWA's local coherence keeps perplexity low, while SA's global coverage lifts downstream tasks, particularly LAMBADA, which requires integrating broad discourse context to predict the final word.

\begin{figure}[htbp]
  \centering
  \includegraphics[width=\columnwidth]{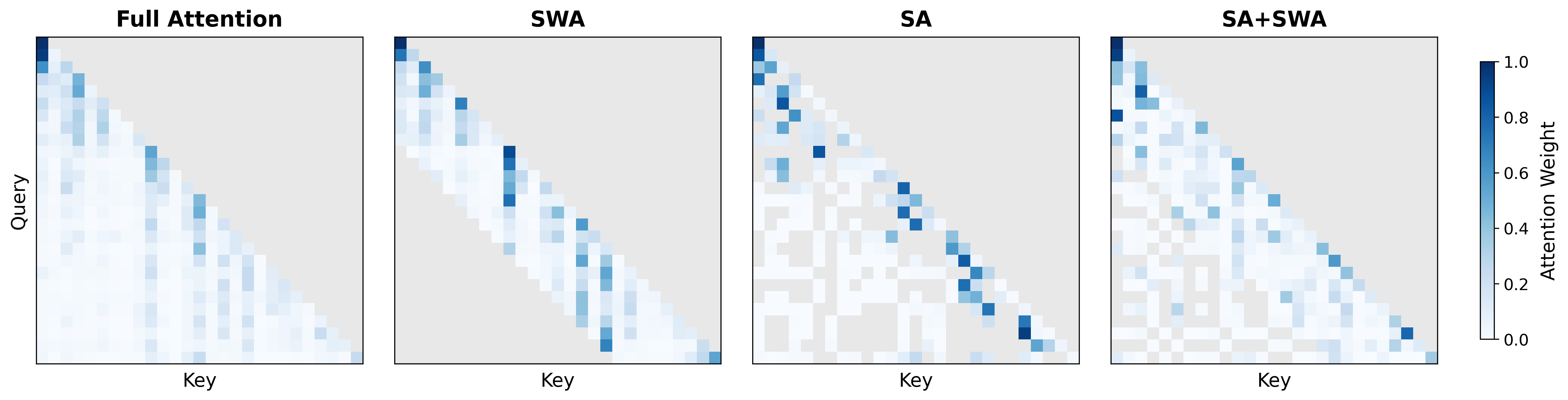}
  \caption{
    Attention weight visualization (Layer 11, Head 0) on a 27-token sequence with window size $w{=}8$. Gray regions are masked (structurally invisible). Blue intensity indicates attention weight. Full Attention exhibits the complete lower-triangular pattern. SWA shows a strict diagonal band with all out-of-window positions masked. Stochastic Attention introduces scattered non-zero entries beyond the diagonal band. These are distant tokens that became local neighbors after random permutation, enabling direct long-range information flow within the same $O(nw)$ budget. SA + SWA combines both patterns: the SWA path provides the dense diagonal band for local coherence, while the SA path adds stochastic long-range connections, with the learned gate adaptively balancing the two.
  }
  \label{fig:attn_patterns}
\end{figure}

To provide further intuition, \Cref{fig:attn_patterns} visualizes the attention patterns of different mechanisms.
SWA produces a strict diagonal band: tokens can only attend within their local window.
Stochastic Attention, by contrast, introduces scattered attention entries far from the diagonal. These correspond to originally distant tokens that became neighbors in the permuted sequence, enabling direct long-range information pathways.
The SA + SWA combination exhibits both the dense diagonal band from SWA and the scattered long-range entries from SA, explaining its strong performance across tasks requiring both local coherence and global reasoning.

\subsection{Training-free inference on Qwen3}
\label{sec:qwen}

To evaluate whether SA can serve as a drop-in replacement for SWA in pretrained LLMs without additional training, we modify the attention mechanism of Qwen3-8B and Qwen3-30B-A3B \citep{Yang2025b} at inference time.
We implement four attention modes sharing the same model weights:
(1)~\textbf{Full}: standard full causal attention (baseline),
(2)~\textbf{SWA}: sliding-window attention with window size $w$,
(3)~\textbf{Stochastic}: SA (random permutation + SWA with the same $w$),
(4)~\textbf{MoBA}: Mixture of Block Attention \citep{Lu2025} with block size $c$ and top-$k$ selection (effective window $\approx c \times k$).
All modes apply only during prefill. Decoding uses full KV-cache attention. For Stochastic mode, RoPE position encodings use the original token positions (not the permuted positions), consistent with the pre-training setup and ensuring compatibility with the model's learned positional representations.
We evaluate on 7 benchmarks using lm-evaluation-harness \citep{eval-harness}: HellaSwag~\citep{zellers2019}, MMLU~\citep{hendrycks2021}, LAMBADA~\citep{paperno2016}, ARC-Easy, ARC-Challenge~\citep{clark2018}, BoolQ (loglikelihood)~\citep{clark2019}, and HumanEval (generation)~\citep{chen2021}.
We sweep the effective window size across $w \in \{16, 32, 64, 128, 256, 512\}$ for SWA and Stochastic. For MoBA, the minimum viable chunk size is 32 (smaller chunks trigger CUDA kernel errors), so we test $c \in \{32, 64, 128, 256\}$ with $k{=}2$ (effective windows 64--512).

Since Qwen3 is trained with full attention, its representations already encode long-range dependencies. SWA at inference time abruptly removes all out-of-window information, creating a distribution shift. SA mitigates this by ensuring that each token can still attend to a random global subset, approximately preserving the full-attention information flow within an $O(nw)$ budget and making it a closer approximation to training-time attention than SWA's strict locality.

\subsubsection{Main results}

\begin{figure*}[htbp]
  \centering
  \includegraphics[width=0.8\textwidth]{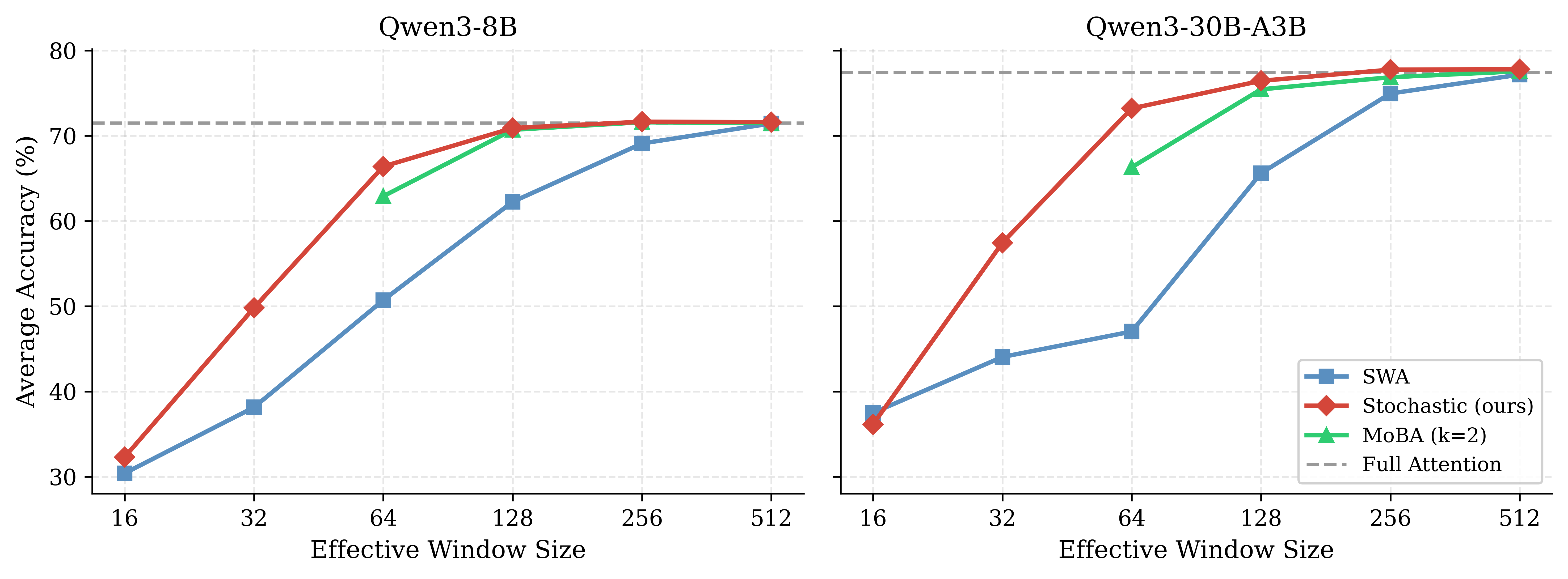}
  \caption{Average accuracy across 7 benchmarks as a function of effective window size for Qwen3-8B (left) and Qwen3-30B-A3B (right). Stochastic Attention (red) recovers the full-attention baseline (dashed gray) most rapidly as window size increases, consistently outpacing SWA (blue) and matching or exceeding MoBA (green) at comparable compute budgets.}
  \label{fig:qwen_scaling}
\end{figure*}

\Cref{fig:qwen_scaling} presents average accuracy as a function of window size for both models. Per-task breakdowns are shown in Figures~\ref{fig:qwen_pertask_8b}--\ref{fig:qwen_pertask_30b}, and full numerical results are provided in Tables~\ref{tab:qwen_main}--\ref{tab:qwen_30b} (Appendix~\ref{sec:appendix_tables}).
Several consistent patterns emerge across model scales.
\textbf{First}, Stochastic recovers full-attention quality fastest: on Qwen3-8B at $w_{\text{eff}} {=} 128$, it already achieves 70.9\% average accuracy (within 1 point of the 71.5\% baseline), while SWA lags at 62.2\%.
The gap is even larger on Qwen3-30B-A3B, where Stochastic reaches 73.2\% at $w_{\text{eff}} {=} 64$ (vs.\ 47.0\% for SWA and 66.3\% for MoBA).
\textbf{Second}, Stochastic consistently outperforms MoBA ($k{=}2$) by 3--7 points at $w_{\text{eff}} {=} 64$ and $128$ across both models, with particularly large gains on MMLU, BoolQ, and LAMBADA.
\textbf{Third}, at very small windows ($w_{\text{eff}} {=} 32$), SWA collapses on knowledge-intensive tasks (MMLU: 29.0 on 8B, 34.9 on 30B), while Stochastic retains substantially higher scores (44.4 / 52.0), confirming effective global information flow even with very local windows.

\begin{figure}[htbp]
  \centering
  \includegraphics[width=\columnwidth]{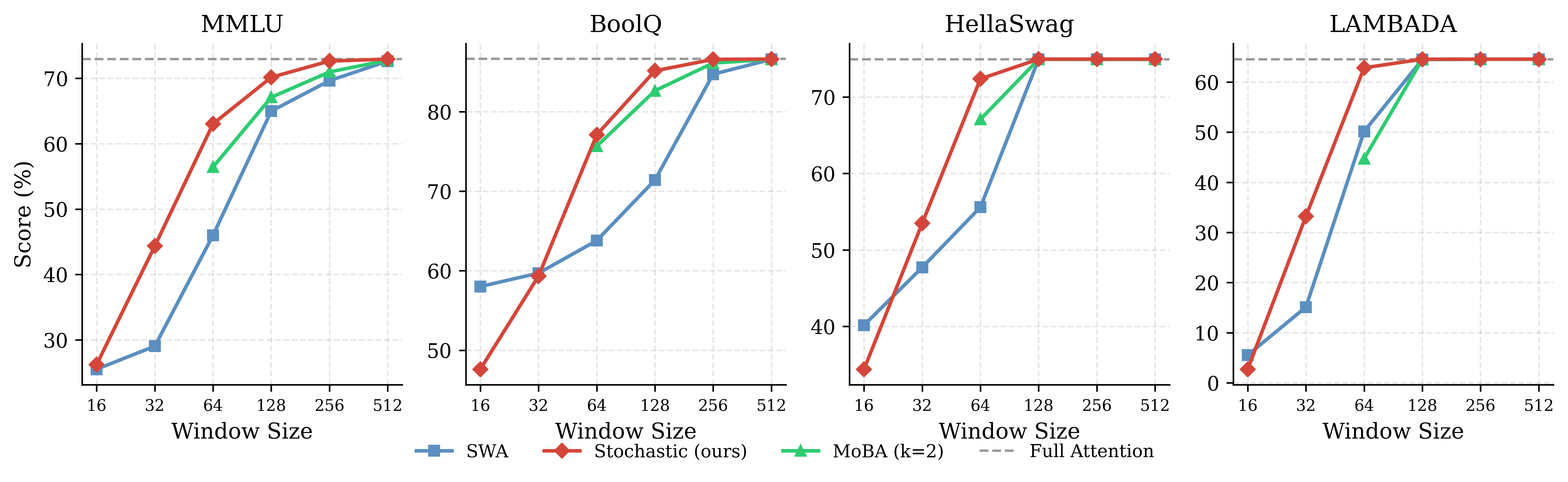}
  \caption{Per-task accuracy vs.\ window size on Qwen3-8B for four representative benchmarks.}
  \label{fig:qwen_pertask_8b}
\end{figure}

\begin{figure}[htbp]
  \centering
  \includegraphics[width=\columnwidth]{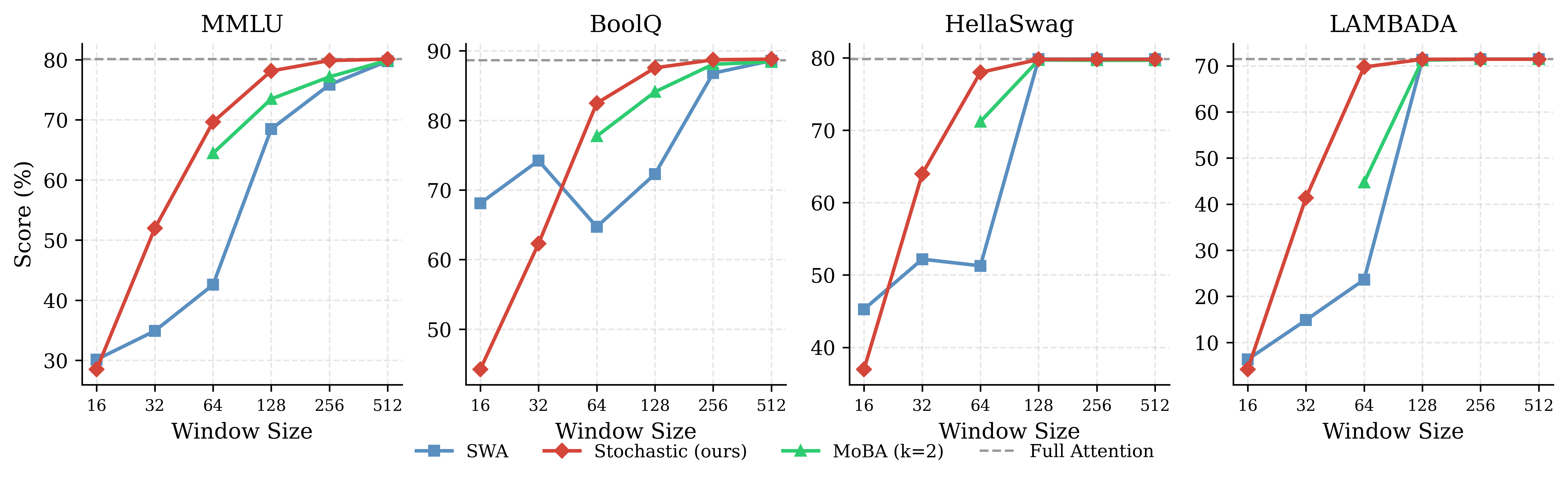}
  \caption{Per-task accuracy vs.\ window size on Qwen3-30B-A3B for four representative benchmarks.}
  \label{fig:qwen_pertask_30b}
\end{figure}

Figures~\ref{fig:qwen_pertask_8b}--\ref{fig:qwen_pertask_30b} show per-task scaling curves on both models. On MMLU and BoolQ, tasks that require aggregating information across contexts, Stochastic converges to the full-attention baseline substantially faster than SWA. The advantage is consistent across both model scales. All Qwen3 results use a single random seed. We did not observe significant variance across preliminary runs with different seeds. Additional per-task results are provided in Appendix~\ref{sec:appendix_pertask}.

\subsection{Efficiency analysis}
\label{sec:efficiency}

We profile attention throughput and memory by isolating the attention computation (forward + backward) at various sequence lengths on a single A100 80GB GPU. Each sequence length is benchmarked in a separate process to avoid compilation interference.

\begin{table}[htbp]
\centering
\begin{tabular}{lccc}
\toprule
\textbf{Seq.\ Length} & \textbf{SA ($w{=}256$)} & \textbf{Full Attn} & \textbf{Speedup} \\
\midrule
2{,}048  & 5.4\,ms   & 8.0\,ms     & 1.5$\times$  \\
4{,}096  & 7.9\,ms   & 27.2\,ms    & 3.5$\times$  \\
8{,}192  & 15.2\,ms  & 99.7\,ms    & 6.6$\times$  \\
16{,}384 & 29.6\,ms  & 379.9\,ms   & 12.8$\times$ \\
32{,}768 & 52.8\,ms  & 1{,}477\,ms & 28.0$\times$ \\
\bottomrule
\end{tabular}
\caption{Attention layer latency (ms, forward+backward) on A100 80GB. SA uses compiled FlexAttention \citep{dong2024} with $w{=}256$. Full attention sets $w{=}L$. Measured with $B{=}16$, $H{=}16$, $d_h{=}64$, bf16.}
\label{tab:efficiency}
\end{table}

Table~\ref{tab:efficiency} reports results from the training sequence length of 2{,}048 onward, where the speedup is stable and meaningful.\footnote{At shorter sequences ($n \leq 1{,}024$), FlexAttention's fixed block-level overhead (128$\times$128 granularity) dominates, making wall-clock comparisons noisy.}
The speedup approximately doubles with each doubling of sequence length (1.5$\times$ at 2K $\to$ 6.6$\times$ at 8K $\to$ 28$\times$ at 32K), consistent with the theoretical $O(nw)$ vs.\ $O(n^2)$ scaling.
For the dual-path SA + SWA configuration, the attention cost is approximately $2\times$ that of single-path SA, but remains $O(nw)$ and retains substantial speedups over full attention at long sequences.

\section{Conclusion}

We have introduced Stochastic Attention (SA), a parameter-free enhancement for sliding-window attention that applies random permutations before windowed attention to transform fixed local windows into stochastic global ones. SA preserves the $O(nw)$ per-layer cost of SWA while achieving exponentially growing receptive fields through depth. When combined with SWA via a lightweight learned gate, the resulting architecture reproduces the small-world regime observed in the fruit fly connectome: dense local clustering from SWA and distributed long-range shortcuts from SA.

Pre-training experiments show the gated SA + SWA combination outperforms both pure SWA and full attention in average downstream accuracy, and training-free inference on Qwen3-8B and Qwen3-30B-A3B demonstrates that SA applied post-hoc to pretrained models matches full-attention quality at a fraction of the compute. Because SWA is already widely deployed in modern foundation models (e.g., Mistral, Gemma~2, gpt-oss), SA can serve as a drop-in upgrade wherever windowed attention layers exist.

More broadly, these results reinforce a lesson from neuroscience: global information flow need not rely on dense all-to-all connectivity, but can emerge from the interplay of structured local computation and sparse long-range shortcuts accumulated through depth.

\section*{Ethics Statement}

This work proposes a general-purpose attention mechanism for Transformer architectures. The method itself does not introduce new ethical risks beyond those inherent to large language models. All experiments use publicly available models (Qwen3) and datasets (SlimPajama, standard NLP benchmarks). No private or sensitive data was used. As with any improvement to language model efficiency or expressivity, downstream applications should be evaluated for potential misuse independently of the architectural contribution.

\section*{Reproducibility Statement}

We provide full architectural and training details in Appendix~\ref{sec:appendix_arch} and Appendix~\ref{sec:appendix_training}, including model dimensions, optimizer hyperparameters, learning rate schedules, batch sizes, and hardware specifications. The SA mechanism requires no additional hyperparameters beyond the window size $w$, which is shared with standard SWA. Pseudocode for both Stochastic Attention and the gated SA + SWA combination is provided in Algorithms~\ref{alg:stochastic_attn}--\ref{alg:fusion}. All proofs and derivations are given in Appendix~\ref{sec:appendix_proofs}. The training-free inference experiments modify only the attention mask of publicly available Qwen3 models and are evaluated using the public lm-evaluation-harness framework \citep{eval-harness}. We will release our implementation upon acceptance.

% \section*{LLM Usage Disclosure}

% In accordance with the COLM policy on the use of large language models, we disclose that LLM-based tools were used during the preparation of this work for grammar checking, proofreading, and refining code implementations (e.g., debugging). All research ideas, experimental design, theoretical analysis, and scientific writing are the original work of the authors. No LLM was used to generate research ideas, produce data or figures, write original content, or perform evaluation.

\newpage
\bibliography{refs}
\bibliographystyle{colm2026_conference}

\newpage
\appendix

\section{Proofs and Derivations}
\label{sec:appendix_proofs}

Throughout this section, we use the circular window convention $\cN_w(i) = \{j : |i-j|_n < w/2\}$ as defined in \S\ref{sec:prelim}.

\subsection{Connection Probability (Eq.~\ref{eq:connection-prob})}
\label{sec:proof_connection}

\begin{proposition}
For a uniform random permutation $\sigma \sim \mathrm{Uniform}(\cS_n)$ and any fixed pair $(i,j)$ with $i \neq j$,
\[
\Pr\bigl[j \in \sigma^{-1}(\cN_w(\sigma(i)))\bigr] = \frac{w-1}{n-1}.
\]
\end{proposition}
\begin{proof}
Since $\sigma$ is uniform, $\sigma(i)$ is uniform over $[n]$. Conditioned on $\sigma(i) = a$, the circular window $\cN_w(a)$ contains exactly $w$ positions (including $a$ itself). The image $\sigma(j)$ is uniform over $[n] \setminus \{a\}$ (the remaining $n-1$ positions). Of these, exactly $w-1$ fall in $\cN_w(a) \setminus \{a\}$. Therefore $\Pr[j \in \sigma^{-1}(\cN_w(\sigma(i))) \mid \sigma(i)=a] = (w-1)/(n-1)$ for every $a$, and marginalizing gives the result.
\end{proof}

\subsection{Receptive Field Expansion}
\label{sec:proof_expansion}

\begin{proposition}
Let $R_\ell(i)$ denote the set of tokens reachable from token $i$ through $\ell$ SA layers with independent permutations, and let $r = |R_\ell(i)|$. Then
\[
\E\bigl[|R_{\ell+1}(i)| \;\big|\; |R_\ell(i)| = r\bigr] \geq r + (n-r)\Bigl[1 - \Bigl(1-\frac{w-1}{n-1}\Bigr)^r\Bigr].
\]
When $rw \ll n$, this implies $\E[|R_{\ell+1}(i)|] = \Omega(rw)$, giving $\E[|R_\ell(i)|] = \Omega((w/4)^\ell)$ and full coverage in $O(\log_w n)$ layers.
\end{proposition}

\begin{proof}
At layer $\ell+1$, a fresh permutation $\sigma_{\ell+1}$ is drawn independently. For any target token $k \notin R_\ell(i)$, the probability that $k$ is not reached by any of the $r$ tokens in $R_\ell(i)$ is:
\[
\Pr[k \notin R_{\ell+1}(i) \mid R_\ell(i)] = \Pr\Bigl[\bigcap_{j \in R_\ell(i)} \{k \notin \sigma_{\ell+1}^{-1}(\cN_w(\sigma_{\ell+1}(j)))\}\Bigr].
\]
All edges share the same permutation $\sigma_{\ell+1}$, so the events are not independent. However, the product bound still holds as an upper bound. To see this, condition on $\sigma_{\ell+1}(k) = s$ for some fixed slot $s$. Given this conditioning, $\sigma_{\ell+1}$ restricted to the remaining $n-1$ tokens is a uniform permutation on $[n] \setminus \{s\}$. For each $j \in R_\ell(i)$, $j$ reaches $k$ iff $\sigma_{\ell+1}(j)$ lands within the window around $s$, i.e., $|\sigma_{\ell+1}(j) - s|_n < w/2$. Since the $\sigma_{\ell+1}(j)$ values for distinct $j$ are drawn without replacement from $[n] \setminus \{s\}$, placing one token near $s$ reduces the number of remaining slots near $s$ for others. This is a negatively correlated sampling scheme, so:
\[
\Pr[k \notin R_{\ell+1}(i) \mid R_\ell(i)] \leq \prod_{j \in R_\ell(i)} \Bigl(1 - \frac{w-1}{n-1}\Bigr) = \Bigl(1 - \frac{w-1}{n-1}\Bigr)^r.
\]
By linearity of expectation over all $n - r$ unreached tokens:
\[
\E[|R_{\ell+1}(i)| \mid |R_\ell(i)|=r] \geq r + (n-r)\Bigl[1 - \Bigl(1-\frac{w-1}{n-1}\Bigr)^r\Bigr].
\]

For the asymptotic bound when $rw \ll n$: let $p = (w-1)/(n-1) \approx w/n$. Using $1-(1-p)^r \geq 1 - e^{-rp} \geq rp(1 - rp/2)$, valid for $rp \leq 1$, and noting $(n-r) \geq n/2$ when $r \leq n/2$:
\[
\E[|R_{\ell+1}(i)|] \geq r + \frac{n}{2} \cdot rp \cdot \frac{1}{2} = r + \frac{rw}{4} = r\Bigl(1 + \frac{w}{4}\Bigr) = \Omega(rw).
\]
Iterating from $|R_0(i)| = 1$ gives $\E[|R_\ell(i)|] = \Omega((w/4)^\ell)$. Full coverage ($R_\ell(i) = [n]$) is achieved when $(w/4)^\ell \geq n$, i.e., $\ell = O(\log n / \log w)$. Note that the base of the exponential is $w/4$ rather than $w$ due to the approximation used. This affects only the constant factor in the coverage depth, not the $O(\log_w n)$ scaling.
\end{proof}

\subsection{Approximation of Full Attention and Variance Bound}
\label{sec:proof_variance}

In the high-temperature limit $\tau \to \infty$, the expected SA output satisfies:
\begin{equation}
\label{eq:unbiased}
\lim_{\tau \to \infty} \E_\sigma\!\left[\mathrm{StoAttn}_\sigma^{(\tau)}(i)\right] = \frac{1}{n} \sum_{j=1}^n V_j + O\!\left(\frac{1}{w}\right),
\end{equation}
making SA an approximately unbiased estimator of uniform full attention (bias $O(1/w)$).

\begin{proposition}
Assuming $\|V_j\| \leq B$ for all $j \in [n]$, the variance of the SA output satisfies:
\[
\E_\sigma\!\Big[\big\|\mathrm{StoAttn}_\sigma(i) - \E_\sigma[\mathrm{StoAttn}_\sigma(i)]\big\|^2\Big] \leq \frac{4B^2}{w}.
\]
\end{proposition}
\begin{proof}
Conditioned on $\sigma$, the SA output is $\mathrm{StoAttn}_\sigma(i) = \sum_{j \in \tilde{\cN}_w^\sigma(i)} \alpha_{ij}^\sigma V_j$, which is a weighted average over $|\tilde{\cN}_w^\sigma(i)| = w$ value vectors. Since $\|\sum_j \alpha_j V_j\| \leq B$ for any convex combination when $\|V_j\| \leq B$, we have $\|\mathrm{StoAttn}_\sigma(i)\| \leq B$.

The variance decomposes as:
\begin{align*}
\mathrm{Var}_\sigma[\mathrm{StoAttn}_\sigma(i)] &= \E_\sigma[\|\mathrm{StoAttn}_\sigma(i)\|^2] - \|\E_\sigma[\mathrm{StoAttn}_\sigma(i)]\|^2 \\
&\leq \E_\sigma[\|\mathrm{StoAttn}_\sigma(i)\|^2] \leq B^2.
\end{align*}

For a tighter bound, observe that the randomness enters through the choice of which $w$ tokens appear in the window. The SA output can be viewed as an importance-weighted sample from the full set of $n$ values. Under uniform attention ($\alpha_{ij} = 1/w$), the output is $\frac{1}{w}\sum_{j \in S} V_j$ where $S$ is a random subset of size $w$. This is a sample mean of $w$ draws without replacement from $\{V_1, \ldots, V_n\}$. By standard results on sampling without replacement, the variance is:
\[
\mathrm{Var}\Bigl[\frac{1}{w}\sum_{j\in S} V_j\Bigr] = \frac{1}{w} \cdot \frac{n-w}{n-1} \cdot \sigma_V^2 \leq \frac{\sigma_V^2}{w},
\]
where $\sigma_V^2 = \frac{1}{n}\sum_{j=1}^n \|V_j - \bar{V}\|^2 \leq 4B^2$. The bound above holds under uniform attention. For finite temperature with data-dependent softmax weights, the effective number of attended tokens may be smaller than $w$ (due to concentration of attention mass), and the variance bound becomes $O(B^2 / w_{\mathrm{eff}})$ where $w_{\mathrm{eff}}$ is the effective window size. In the worst case of fully concentrated attention ($w_{\mathrm{eff}} = 1$), the variance is $O(B^2)$.
\end{proof}

\subsection{Spectral Mixing: Single Layer vs.\ Multi-Layer}
\label{sec:proof_spectral}

Under uniform attention, the transition matrix for a single SA layer is $\bA^\sigma = \frac{1}{w}\bP_\sigma^\top A_w \bP_\sigma$, where $A_w$ is the adjacency matrix of the circulant $C_{n,w}$. Since $\bP_\sigma$ is an orthogonal (permutation) matrix, $\bA^\sigma$ is similar to $A_w/w$ and has identical eigenvalues. In particular, $|\lambda_2(\bA^\sigma)| = |\lambda_2(A_w/w)|$ for every $\sigma$, giving the same single-layer spectral gap as SWA: $O(w^2/n^2)$.

The advantage emerges through multi-layer composition.
For $L$ SWA layers, the composed transition matrix is simply $(A_w/w)^L$, inheriting the slow spectral gap of the circulant. For $L$ SA layers with independent permutations, the composed matrix is $\bA^{(1:L)} = \bA^{\sigma_L} \cdots \bA^{\sigma_1} = \frac{1}{w^L}\bP_{\sigma_L}^\top A_w \bP_{\sigma_L} \cdots \bP_{\sigma_1}^\top A_w \bP_{\sigma_1}$. Crucially, this product is not similar to $(A_w/w)^L$ because the conjugating permutations differ across layers.

The receptive field expansion result (\Cref{sec:proof_expansion}) implies that the reachability graph after $L$ SA layers is an expander with high probability when $L = O(\log_w n)$: starting from any token, $\Omega(w^L)$ tokens are reachable. This corresponds to rapid mixing of the composed random walk, in contrast to the $O(n/w)$ layers required for SWA. The key insight is that independent permutations at each layer prevent the slow eigenmodes of the circulant from persisting across depth. \Cref{tab:spectral} summarizes the comparison.

\begin{table}[htbp]
\centering
\renewcommand{\arraystretch}{1.2}
\begin{tabular}{@{}lccc@{}}
\toprule
\textbf{Mechanism} & \textbf{Cost / Layer} & \textbf{Receptive Field} & \textbf{Coverage Depth} \\
\midrule
Full Attention & $O(n^2)$ & $n$ (layer 1) & $1$ \\
Sliding Window & $O(nw)$ & $\ell w$ & $O(n/w)$ \\
\textbf{SA (ours)} & $O(nw)$ & $\Omega((w/4)^\ell)$ & $\bm{O(\log_w n)}$ \\
\textbf{SA + SWA} & $2 \cdot O(nw)$ & $\Omega((w/4)^\ell)$ & $\bm{O(\log_w n)}$ \\
\bottomrule
\end{tabular}
\caption{Comparison of attention mechanisms. $n$: sequence length, $w$: window size, $\ell$: number of layers.}
\label{tab:spectral}
\end{table}

\subsection{Bias-Variance Decomposition}
\label{sec:proof_bv}

Let $\bY^*_i = \mathrm{FullAttn}(\bQ, \bK, \bV)_i$ denote the full attention output. Conditioning on $\bX$ (so that $g^{\mathrm{sa}}_i$, $g^{\mathrm{swa}}_i$, and $\bY^{\mathrm{swa}}_i$ are all deterministic), the MSE decomposes as:
\begin{equation}
\label{eq:fusion_bv}
\E_\sigma\!\Big[\big\|\bY_i - \bY^*_i\big\|^2\Big] = \underbrace{\big\|g^{\mathrm{sa}}_i \odot b^{\mathrm{sa}}_i + g^{\mathrm{swa}}_i \odot b^{\mathrm{swa}}_i\big\|^2}_{\text{bias}^2} + \underbrace{\big\|g^{\mathrm{sa}}_i\big\|^2 \cdot v^{\mathrm{sa}}_i}_{\text{variance}}\,.
\end{equation}
\begin{proof}
Write:
\begin{align*}
\bY_i - \bY^*_i &= g^{\mathrm{sa}}_i \odot (\bY^{\mathrm{sa}}_i - \bY^*_i) + g^{\mathrm{swa}}_i \odot (\bY^{\mathrm{swa}}_i - \bY^*_i) \\
&= g^{\mathrm{sa}}_i \odot \bigl[(\bY^{\mathrm{sa}}_i - \E[\bY^{\mathrm{sa}}_i]) + b^{\mathrm{sa}}_i\bigr] + g^{\mathrm{swa}}_i \odot b^{\mathrm{swa}}_i,
\end{align*}
where $b^{\mathrm{sa}}_i = \E_\sigma[\bY^{\mathrm{sa}}_i] - \bY^*_i$ and $b^{\mathrm{swa}}_i = \bY^{\mathrm{swa}}_i - \bY^*_i$. Taking $\E_\sigma[\|\cdot\|^2]$ and using the fact that the zero-mean term $\bY^{\mathrm{sa}}_i - \E[\bY^{\mathrm{sa}}_i]$ is uncorrelated with the deterministic bias terms:
\[
\E_\sigma[\|\bY_i - \bY^*_i\|^2] = \|g^{\mathrm{sa}}_i \odot b^{\mathrm{sa}}_i + g^{\mathrm{swa}}_i \odot b^{\mathrm{swa}}_i\|^2 + \|g^{\mathrm{sa}}_i\|^2 \cdot v^{\mathrm{sa}}_i,
\]
where $v^{\mathrm{sa}}_i = \E_\sigma[\|\bY^{\mathrm{sa}}_i - \E[\bY^{\mathrm{sa}}_i]\|^2]$. The variance term follows from $\E[\|a \odot X\|^2] = \sum_k a_k^2 \E[X_k^2]$, which holds exactly since the components of $X = \bY^{\mathrm{sa}}_i - \E[\bY^{\mathrm{sa}}_i]$ are uncorrelated with the deterministic vector $a = g^{\mathrm{sa}}_i$. If the per-component variances $\E[X_k^2]$ are approximately uniform across dimensions ($\E[X_k^2] \approx v^{\mathrm{sa}}_i / d$), this simplifies to $\|g^{\mathrm{sa}}_i\|^2 \cdot v^{\mathrm{sa}}_i / d$. The expression in Eq.~\ref{eq:fusion_bv} uses this approximation, absorbing the $1/d$ factor into the definition of $v^{\mathrm{sa}}_i$.
\end{proof}

\section{Model Architecture Details}
\label{sec:appendix_arch}

All pre-training models follow a decoder-only Transformer++ architecture (RMSNorm, SwiGLU, RoPE, no bias), with the components described below. The four model variants (Full Attention, SWA, SA, SA+SWA) differ only in the attention mechanism. All other components are identical.

\textbf{Embedding.}
We use a learned token embedding of dimension $d = 1024$ with vocabulary size 32,000 (Mistral tokenizer). The output LM head shares weights with the input embedding (tied embeddings).

\textbf{Transformer layers.}
The model consists of 24 identical layers, each containing:
\begin{itemize}[leftmargin=1.5em,itemsep=2pt]
  \item \textbf{Pre-norm.} RMSNorm is applied before both the attention and MLP sub-layers.
  \item \textbf{Attention.} Multi-head attention with 16 heads ($d_h = 64$). The Q, K, V projections are fused into a single linear layer ($d \to 3d$, no bias), followed by an output projection ($d \to d$, no bias). RoPE \citep{Su2023} is applied to Q and K using the tokens' original sequence positions (prior to any shuffling). An attention gate ($d \to d$, sigmoid) modulates the attention output before projection: $\mathrm{gate}(\bY_{\mathrm{attn}}) \odot \bY_{\mathrm{attn}}$.
  \item \textbf{MLP.} SwiGLU activation with hidden dimension $\lfloor 2.67 \times d \rfloor = 2{,}734$, implemented as three linear layers: gate ($d \to d_{\mathrm{ff}}$), up ($d \to d_{\mathrm{ff}}$), and down ($d_{\mathrm{ff}} \to d$), all without bias.
\end{itemize}

\textbf{Attention variants.}
\begin{itemize}[leftmargin=1.5em,itemsep=2pt]
  \item \textbf{Full Attention} (360M params): standard causal attention with window size $w = L$ (full sequence).
  \item \textbf{SWA} (360M params): causal sliding-window attention with $w = 256$.
  \item \textbf{SA} (360M params): causal sliding-window attention in shuffled space with $w = 256$. A fresh random permutation is sampled independently for each layer (shared across heads) at each training step.
  \item \textbf{SA+SWA} (385M params): dual-path architecture where the single attention gate is replaced by two independent gates ($\mathrm{gate}_{\mathrm{local}}$, $\mathrm{gate}_{\mathrm{global}}$), each $d \to d$ with sigmoid, producing the fused output $\mathrm{gate}_{\mathrm{local}}(\bY^{\mathrm{swa}}) \odot \bY^{\mathrm{swa}} + \mathrm{gate}_{\mathrm{global}}(\bY^{\mathrm{sa}}) \odot \bY^{\mathrm{sa}}$. This adds ${\sim}$25M parameters (${\sim}$1.05M $\times$ 24 layers) compared to the single-path variants.
\end{itemize}

\textbf{Stochastic attention mask.}
In the SA and SA+SWA variants, the attention mask at each layer is constructed as the intersection of two constraints: (1) causal in original space: $\mathrm{pos}(q) \geq \mathrm{pos}(k)$, where $\mathrm{pos}(\cdot)$ denotes the original sequence position, and (2) window in shuffled space: $|\sigma(i) - \sigma(j)|_n < w/2$, where $\sigma$ is the layer-specific random permutation. This mask is implemented efficiently via FlexAttention \citep{dong2024}.

\textbf{Pseudocode.}

\begin{algorithm}[htbp]
\caption{Stochastic Attention (Single Head)}
\label{alg:stochastic_attn}
\KwIn{$\bQ, \bK, \bV \in \R^{n \times d_h}$; window size $w$}
\KwOut{$\bY^{\mathrm{sto}} \in \R^{n \times d_h}$}
Sample $\sigma \sim \mathrm{Uniform}(\cS_n)$\;
$\tilde{\bQ} \leftarrow \bP_\sigma \bQ$; \quad $\tilde{\bK} \leftarrow \bP_\sigma \bK$; \quad $\tilde{\bV} \leftarrow \bP_\sigma \bV$\;
$\tilde{\bY} \leftarrow \mathrm{SWA}(\tilde{\bQ}, \tilde{\bK}, \tilde{\bV}; w)$\;
$\bY^{\mathrm{sto}} \leftarrow \bP_{\sigma^{-1}} \tilde{\bY}$\;
\Return $\bY^{\mathrm{sto}}$\;
\end{algorithm}

\begin{algorithm}[htbp]
\caption{Gated SA + SWA}
\label{alg:fusion}
\KwIn{$\bX \in \R^{n \times d}$; window size $w$; gate parameters $W_g^{\mathrm{swa}}, W_g^{\mathrm{sa}} \in \R^{d \times d}$}
\KwOut{$\bY \in \R^{n \times d}$}
Compute $\bQ, \bK, \bV$ from $\bX$\;
\tcp{Stochastic Attention}
$\bY^{\mathrm{sa}} \leftarrow \textsc{StochasticAttn}(\bQ, \bK, \bV; w)$ \tcp*{\Cref{alg:stochastic_attn}}
\tcp{Sliding-Window Attention}
$\bY^{\mathrm{swa}} \leftarrow \mathrm{SWA}(\bQ, \bK, \bV; w)$\;
\tcp{Gated Combination}
$g^{\mathrm{swa}} \leftarrow \sigmoid(W_g^{\mathrm{swa}}\, (\bY^{\mathrm{swa}})^\top)^\top$; \quad $g^{\mathrm{sa}} \leftarrow \sigmoid(W_g^{\mathrm{sa}}\, (\bY^{\mathrm{sa}})^\top)^\top$\;
$\bY \leftarrow g^{\mathrm{sa}} \odot \bY^{\mathrm{sa}} + g^{\mathrm{swa}} \odot \bY^{\mathrm{swa}}$\;
\Return $\bY$\;
\end{algorithm}

\newpage
\section{Pre-training Setup Details}
\label{sec:appendix_training}

The models follow a decoder-only Transformer layout with 24 layers, hidden dimension $d = 1024$, 16 attention heads ($d_h = 64$), SwiGLU feed-forward networks (expansion ratio 2.67), RMSNorm, and RoPE \citep{Su2023}.
All models use window size $w = 256$. In SA layers, a random permutation is applied before windowed attention and inverted afterward, with RoPE using the original (pre-permutation) position indices.

We tokenize with the Mistral tokenizer \citep{Jiang2023} (vocabulary size 32{,}000) and train with AdamW ($\beta_1 {=} 0.9$, $\beta_2 {=} 0.95$) at peak learning rate $3 {\times} 10^{-4}$ with cosine decay to $3 {\times} 10^{-5}$ after linear warmup over 0.5B tokens.
Training uses 4$\times$A100 80GB GPUs with per-GPU batch size 16, gradient accumulation over 30 steps, and sequence length 2048, yielding ${\sim}$3.9M tokens per optimizer step in \texttt{bf16} mixed precision.

\section{Per-Task Scaling Results}
\label{sec:appendix_pertask}

\Cref{fig:appendix_8b,fig:appendix_30b} provide per-task accuracy as a function of effective window size for all 7 evaluated benchmarks on Qwen3-8B and Qwen3-30B-A3B, respectively.

\begin{figure*}[htbp]
  \centering
  \includegraphics[width=\textwidth]{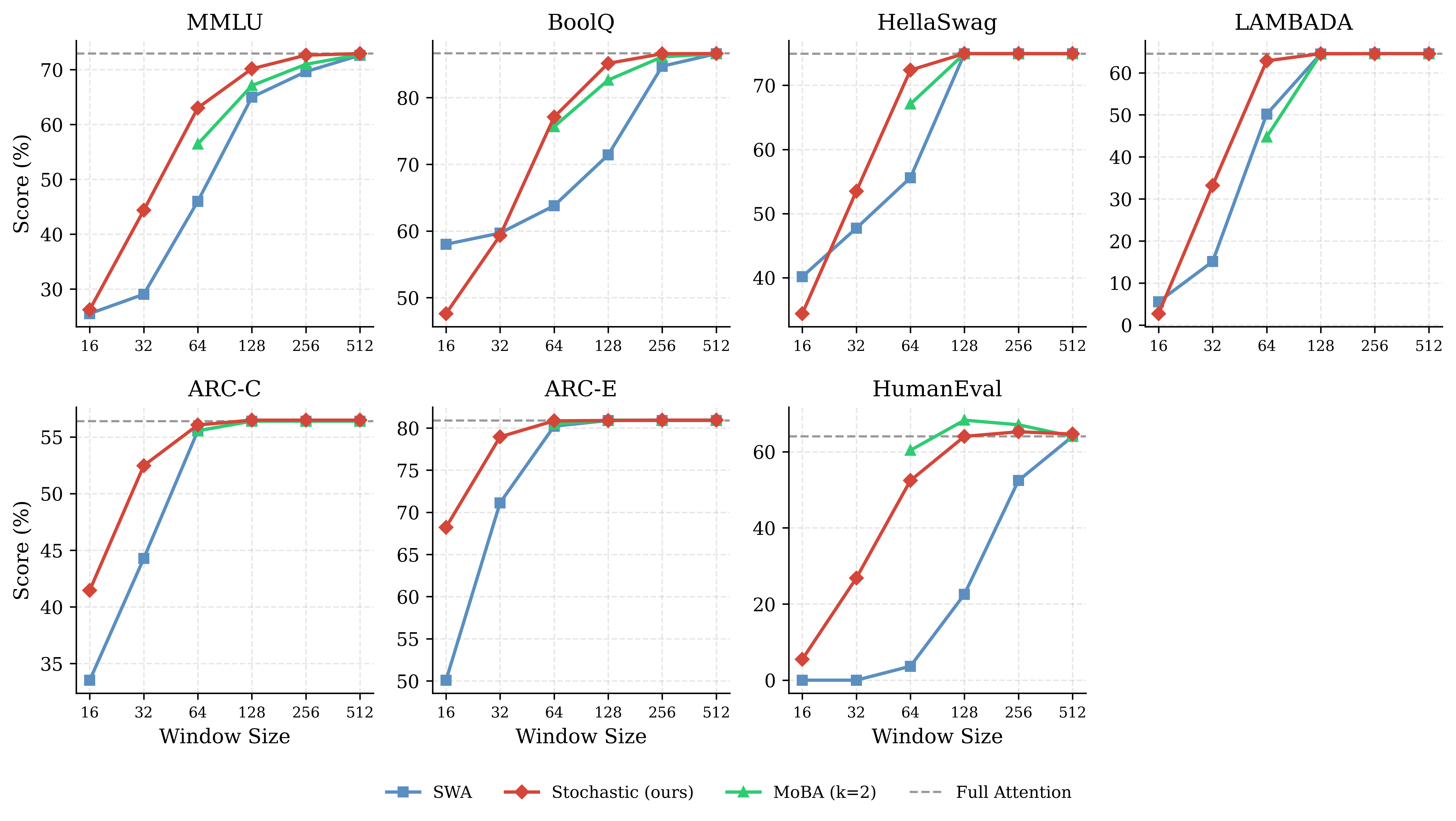}
  \caption{Per-task accuracy vs.\ effective window size for Qwen3-8B across all evaluated benchmarks.}
  \label{fig:appendix_8b}
\end{figure*}

\begin{figure*}[htbp]
  \centering
  \includegraphics[width=\textwidth]{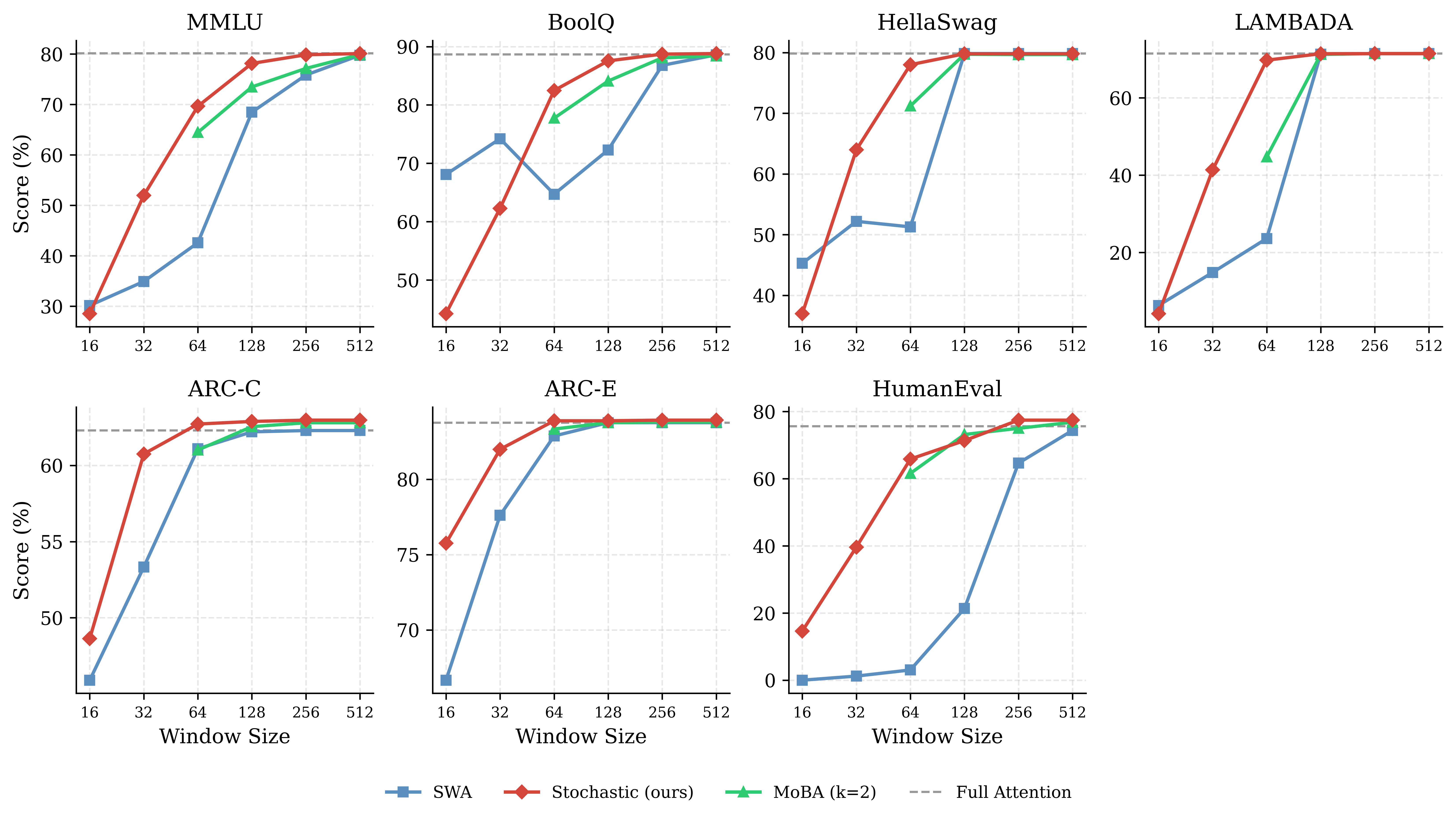}
  \caption{Per-task accuracy vs.\ effective window size for Qwen3-30B-A3B across all evaluated benchmarks.}
  \label{fig:appendix_30b}
\end{figure*}

\section{Detailed Numerical Results}
\label{sec:appendix_tables}

\begin{table}[htbp]
\centering
\resizebox{\columnwidth}{!}{%
\begin{tabular}{ll ccccccc c}
\toprule
\textbf{$w_{\text{eff}}$} & \textbf{Mode} & \textbf{Hella.} & \textbf{MMLU} & \textbf{ARC-C} & \textbf{BoolQ} & \textbf{LMB} & \textbf{ARC-E} & \textbf{HuEval} & \textbf{Avg.} \\
\midrule
$\infty$ & Full & 74.9 & 73.0 & 56.4 & 86.6 & 64.6 & 80.9 & 64.0 & 71.5 \\
\midrule
\multirow{2}{*}{32}
 & SWA        & 47.7 & 29.0 & 44.3 & \textbf{59.7} & 15.1 & 71.1 &  0.0 & 38.1 \\
 & Stochastic & \textbf{53.5} & \textbf{44.4} & \textbf{52.5} & \underline{59.3} & \textbf{33.2} & \textbf{79.0} & \textbf{26.8} & \textbf{49.8} \\
\midrule
\multirow{3}{*}{64}
 & SWA              & 55.6 & 46.0 & \underline{55.6} & 63.8 & 50.2 & 80.2 &  3.7 & 50.7 \\
 & Stochastic       & \textbf{72.4} & \textbf{63.0} & \textbf{56.1} & \textbf{77.1} & \textbf{62.9} & \textbf{80.9} & \underline{52.4} & \textbf{66.4} \\
 & MoBA ($k{=}2$)   & \underline{67.1} & \underline{56.4} & \underline{55.6} & \underline{75.7} & \underline{44.7} & \underline{80.5} & \textbf{60.4} & \underline{62.9} \\
\midrule
\multirow{3}{*}{128}
 & SWA              & \underline{74.9} & 65.0 & 56.4 & 71.4 & \underline{64.5} & \textbf{80.9} & 22.6 & 62.2 \\
 & Stochastic       & \textbf{75.0} & \textbf{70.2} & \textbf{56.5} & \textbf{85.1} & \textbf{64.6} & \textbf{80.9} & \underline{64.0} & \textbf{70.9} \\
 & MoBA ($k{=}2$)   & \underline{74.9} & \underline{67.1} & 56.4 & \underline{82.6} & \underline{64.5} & \textbf{80.9} & \textbf{68.3} & \underline{70.7} \\
\midrule
\multirow{3}{*}{256}
 & SWA              & \underline{74.9} & 69.6 & 56.4 & 84.7 & \textbf{64.6} & \textbf{80.9} & 52.4 & 69.1 \\
 & Stochastic       & \textbf{75.0} & \textbf{72.6} & \textbf{56.5} & \textbf{86.6} & \textbf{64.6} & \textbf{80.9} & \underline{65.2} & \textbf{71.6} \\
 & MoBA ($k{=}2$)   & \underline{74.9} & \underline{71.0} & 56.4 & \underline{86.1} & \textbf{64.6} & \textbf{80.9} & \textbf{67.1} & \underline{71.6} \\
\bottomrule
\end{tabular}
}
\caption{Training-free inference on Qwen3-8B. We report accuracy on 7 benchmarks at selected window sizes. Best result per column among efficient methods in \textbf{bold}, \underline{underlined} denotes second best.}
\label{tab:qwen_main}
\end{table}

\begin{table}[htbp]
\centering
\resizebox{\columnwidth}{!}{%
\begin{tabular}{ll ccccccc c}
\toprule
\textbf{$w_{\text{eff}}$} & \textbf{Mode} & \textbf{Hella.} & \textbf{MMLU} & \textbf{ARC-C} & \textbf{BoolQ} & \textbf{LMB} & \textbf{ARC-E} & \textbf{HuEval} & \textbf{Avg.} \\
\midrule
$\infty$ & Full & 79.8 & 80.1 & 62.3 & 88.7 & 71.5 & 83.8 & 75.6 & 77.4 \\
\midrule
\multirow{2}{*}{32}
 & SWA        & 52.2 & 34.9 & 53.3 & \textbf{74.2} & 14.9 & 77.6 &  1.2 & 44.0 \\
 & Stochastic & \textbf{64.0} & \textbf{52.0} & \textbf{60.8} & 62.3 & \textbf{41.4} & \textbf{82.0} & \textbf{39.6} & \textbf{57.4} \\
\midrule
\multirow{3}{*}{64}
 & SWA              & 51.3 & 42.6 & 61.1 & 64.7 & 23.6 & 82.9 &  3.0 & 47.0 \\
 & Stochastic       & \textbf{78.0} & \textbf{69.6} & \textbf{62.7} & \textbf{82.5} & \textbf{69.7} & \textbf{83.9} & \textbf{65.9} & \textbf{73.2} \\
 & MoBA ($k{=}2$)   & \underline{71.2} & \underline{64.4} & \underline{61.0} & \underline{77.7} & \underline{44.7} & \underline{83.3} & \underline{61.6} & \underline{66.3} \\
\midrule
\multirow{3}{*}{128}
 & SWA              & \underline{79.8} & 68.5 & 62.2 & 72.3 & \underline{71.3} & 83.8 & 21.3 & 65.6 \\
 & Stochastic       & \textbf{79.8} & \textbf{78.1} & \textbf{62.9} & \textbf{87.6} & \textbf{71.4} & \textbf{83.9} & \textbf{71.3} & \textbf{76.4} \\
 & MoBA ($k{=}2$)   & 79.7 & \underline{73.5} & \underline{62.5} & \underline{84.1} & 71.2 & 83.8 & \underline{73.2} & \underline{75.4} \\
\midrule
\multirow{3}{*}{256}
 & SWA              & \underline{79.8} & 75.8 & 62.3 & 86.8 & \textbf{71.5} & 83.8 & 64.6 & 74.9 \\
 & Stochastic       & \textbf{79.8} & \textbf{79.9} & \textbf{63.0} & \textbf{88.7} & 71.4 & \textbf{83.9} & \textbf{77.4} & \textbf{77.7} \\
 & MoBA ($k{=}2$)   & 79.7 & \underline{77.2} & \underline{62.8} & \underline{88.1} & \textbf{71.5} & 83.8 & \underline{75.0} & \underline{76.9} \\
\bottomrule
\end{tabular}
}
\caption{Training-free inference on Qwen3-30B-A3B. Same setup as \Cref{tab:qwen_main}.}
\label{tab:qwen_30b}
\end{table}

\end{document}